\documentclass[twoside,11pt]{article}

\usepackage{jmlr2e}
\usepackage{array}

\usepackage{amsmath}
\usepackage{color}
\usepackage[english,ruled,vlined]{algorithm2e}

\usepackage[T1]{fontenc}

\usepackage{textcomp}

\newcommand{\diag}{\textup{diag}}
\newcommand{\KL}{\textup{KL}}
\newcommand{\Tr}{\textup{Tr}}

\newcommand{\argmax}{\textup{argmax}}
\newcommand{\argmin}{\textup{argmin}}

\ShortHeadings{Globally Sparse Probabilistic PCA}{Globally Sparse Probabilistic PCA}
\firstpageno{1}

\begin{document}

\title{\textsc{Bayesian Variable Selection \\ for  Globally Sparse Probabilistic PCA}}


 \author{\name Charles Bouveyron \email \emph{charles.bouveyron@parisdescartes.fr} \\  
     \addr Laboratoire MAP5, UMR CNRS 8145\\
      Universit\'e Paris Descartes \& Sorbonne Paris Cité\\
          \name Pierre Latouche \email \emph{pierre.latouche@univ-paris1.fr} \\
    \addr Laboratoire SAMM, EA 4543\\
     Universit\'e Paris 1 Panth\'eon-Sorbonne\\
   \name Pierre-Alexandre Mattei \email \emph{pierre-alexandre.mattei@parisdescartes.fr} \\
           \addr Laboratoire MAP5, UMR CNRS 8145\\
  Universit\'e Paris Descartes \& Sorbonne Paris Cité}

\maketitle

\vskip 0.2in

\begin{abstract}
Sparse versions of principal component analysis (PCA) have imposed themselves as simple, yet powerful ways of selecting relevant features of high-dimensional data in an unsupervised manner. However, when several sparse principal components are computed, the interpretation of the selected variables is difficult since each axis has its own sparsity pattern and has to be interpreted separately. To overcome this drawback, we propose a Bayesian procedure called globally sparse probabilistic PCA (GSPPCA) that allows to obtain several sparse components with the same sparsity pattern. This allows the practitioner to identify the original variables which are relevant to describe the data. To this end, using Roweis' probabilistic interpretation of PCA and a Gaussian prior on the loading matrix, we provide the first exact computation of the marginal likelihood of a Bayesian PCA model. To avoid the drawbacks of discrete model selection, a simple relaxation of this framework is presented. It allows to find a path of models using a variational expectation-maximization algorithm. The exact marginal likelihood is then maximized over this path. This approach is illustrated on real and synthetic data sets. In particular, using unlabeled microarray data, GSPPCA infers much more relevant gene subsets than traditional sparse PCA algorithms.
\end{abstract}


\section{Introduction}
From the children test results of the seminal paper of~\cite{hotelling} to the challenging analysis of microarray data~\citep{ringner2008} {and the recent successes of deep learning \citep{pcanet}}, principal component analysis (PCA) has become one of the most popular tools for data-preprocessing and dimension-reduction. The original procedure consists in projecting the data onto a "principal" subspace spanned by the leading eigenvectors of the sample covariance matrix. It was later shown that this subspace could also be retrieved from the maximum-likelihood estimator of a parameter, in a particular factor analysis model called probabilisitic PCA (PPCA)~\citep{roweis1998,tipping1999}. This probabilistic framework led to diverse Bayesian analysis of PCA~\citep{bishop1999,minka2000,nakajima2011}.

\subsection{Local and global sparsity}

A potential drawback of PCA is that the principal components  are linear combinations of every single original variable, and can therefore be difficult to interpret. To tackle this issue, several procedures have been designed to project the data onto subspaces generated by sparse vectors while retaining as much variance as possible. Many of them were based on convex or partially convex relaxations of cardinality-constrained PCA problems -- among these techniques are the popular $\ell_1$-based SPCA algorithm of \cite{zou2006} or the semidefinite relaxation of \cite{aspremont2008}.
Another strategy is to use a sparsity-inducing prior distributions on the coefficients of the projection matrix~\citep{archambeau2009,guan2009,khanna2015}.

However, when several principal components are computed, these various techniques do not enforce them to have the same sparsity pattern, and each component has to be interpreted individually. While individual interpretation is particularly natural in several cases -- when PCA serves visualization, for example --, it is not adapted to situations where the practitioner aims at \emph{globally} selecting which features are relevant. In these situations, a simple and popular approach has been to consider that the relevant variables correspond to the sparsity pattern of the first principal component~\citep{zou2006,zhang2012}. However, this procedure is limited, and several important aspects of the data may lie in the next principal components. For example, in the colon cancer data set studied by~\cite{aspremont2008}, the most relevant genes were the ones selected not by the first but by the \emph{second} principal component. Another motivation for global sparsity is the fact that, in many real-life situations, the sparsity pattern of the axes computed by a sparse PCA algorithm are extremely close. This is for example the case of the three axes of the template attacks application considered by~\cite{archambeau2009}. {In this setting, forcing these patterns to be equal will give the practitioner a precise idea of which variables are relevant.} Another interesting feature of global sparsity is the fact that, once the common sparsity pattern has been determined, performing PCA on the relevant variables yields orthogonal and uncorrelated principal components -- conversely to most sparse PCA procedures.

\subsection{Related work}

Since the seminal papers of~\cite{jolliffe1972,jolliffe1973} {and \cite{robert1976}}, several methods have been designed to discard features in PCA (see e.g.~\cite{brusco2014} for a recent review). However, these techniques were designed to eliminate redundant, rather that irrelevant variables, and are based on combinatorial algorithms that are not really suitable for high-dimensional problems.

A simple and scalable way of performing variable selection for PCA is to simply keep the features that have the largest marginal variance. In certain cases, this technique is theoretically sound, and was applied for instance to the analysis of electrocardiogram (ECG) data \citep{johnstone2009}. \cite{zhang2011} also proved that it could be used as an efficient preprocessing technique to reduce the dimensionality of ultra-high dimensional problems before applying a traditional sparse PCA algorithm. However, this technique has two main drawbacks. First, it is not robust to simple transformations of the data since simply multiplying a variable by a constant may wrongfully select (or discard) it. An unfortunate consequence of this is the fact that this technique can not be applied to scaled data. Moreover, since it ignores non-marginal information, this technique will behave badly  in the case of correlated features.

A more refined approach to global sparsity is $\ell_1$-based regularization, which has imposed itself as one of the most versatile and efficient approaches to sparse statistical learning \citep{hastie2015}. In a context of \emph{structured} sparse PCA, \cite{jenatton2009} proposed to recast sparse PCA as a penalized matrix factorization problem and suggested that limiting the number of sparsity patterns allowed within the principal vectors could improve the feature extraction quality -- particularly in face recognition problems. Using the $\ell_1-\ell_2$ norm, they derived an algorithm (hereafter referred as SSPCA) that allows to compute $d$ sparse components with exactly $m\leq d$ sparsity patterns. However, they only considered cases where $m$ is larger than $2$ and therefore did not focus on global sparsity. They were followed by \cite{khan2015} who, in a very close framework, argued that global sparsity (which they called \emph{joint sparsity}) led to better representations of hyperspectral images. Other similar approaches based on structured composite norms have been conducted by \cite{masaeli2010}, \cite{gu2011} and \cite{xiaoshuang2013}. \cite{ulfarsson2008,ulfarsson2011} used sparsity inducing penalties together with a PPCA model to enforce global sparsity. They proposed an algorithm called \emph{sparse variable noisy PCA} (hereafter refered as svnPCA) and fixed the amount of penalization using the Bayesian information criterion (BIC) of \cite{schwarz1978}.

{Eventually, it is worth mentioning that global sparsity has also been investigated in other contexts, such as partial least squares regression \citep{liu2013} or electroencephalography (EEG) imaging \citep{wipf2009,gramfort2013}.}

{\subsection{Contributions and organization of the paper}}

{We present in Section 2 a Bayesian approach that allows to project the data onto a \emph{globally sparse subspace} (i.e a subspace spanned by vectors with the same sparsity pattern) while preserving a large part of the variance. To this end, we use the noiseless PPCA model introduced by~\cite{roweis1998} together with an isotropic gaussian prior on the projection matrix and a binary vector that segregates relevant from irrelevant variables. While past Bayesian PCA frameworks relied on variational~\citep{bishopvar,archambeau2009,guan2009} or Laplace~\citep{bishop1999,minka2000} methods to approximate the marginal likelihood, we derive here a closed-form expression for the evidence based on the multivariate Bessel distribution. In order to avoid the drawbacks of discrete model selection and to treat high-dimensional data, we also present a relaxation of our model by replacing the binary vector with a continuous one. Inference of this relaxed model can be performed using a variational expectation-maximization (VEM) algorithm. Such a procedure allows to find a path of models. The exact evidence is eventually maximized over this path, relying on Occam's razor~\citep[chap. 28]{mackay2003}, to select the relevant variables.}

{We illustrate the behaviour of our algorithm and compare it to other methods in Section 3. In particular, we show that Bayesian model selection empirically outperforms $\ell_1-\ell_2$-based regularization on a series of tasks.}

{Sections 4 and 5 are devoted to two applications showcasing the features of our method. The first one concerns signal denoising with wavelets, and shows how global sparsity can surpass traditional sparse PCA algorithms within this context. The second one treats about unsupervised gene selection. Given an (unlabeled) microarray data matrix, we show how GSPPCA can select biologically relevant subsets of genes. Interestingly, we exhibit an important correlation between our exact marginal likelihood expression and a criterion of biological relevance based on pathway enrichment.}

Note that this paper is an extended version of previous work \citep{mattei2016} published in the Proceedings of the $19^\textup{th}$ Conference on Artificial Intelligence and Statistics.

\section{Bayesian variable selection for PCA}
Let us assume that a centered i.i.d. sample $\mathbf{x}_1,...,\mathbf{x}_n \in \mathbb{R}^p$ is observed which one wishes to project onto a $d$-dimensional subspace while retaining as much variance as possible. All the observations are stored in the $n\times p$ matrix $\mathbf{X}=(\mathbf{x}_1,...,\mathbf{x}_n)^T$.

\subsection{Probabilistic PCA}

The PPCA model assumes that each observation is driven by the following generative model
\begin{equation} \label{modelePPCA}
\mathbf{x} = \mathbf{W} \mathbf{y} + \boldsymbol{\varepsilon},
\end{equation} where  $\mathbf{y} \sim \mathcal{N}(0,\mathbf{I}_d)$ is a low-dimensional Gaussian latent vector, $\mathbf{W}$ is a $p \times d$ parameter matrix called the \emph{loading matrix} and  $\boldsymbol{\varepsilon} \sim \mathcal{N}(0, \sigma^2 \mathbf{I}_p)$ is a Gaussian noise term.

This model is a particular instance of factor analysis and was first introduced by \cite{lawley1953}. Following \cite{theobald}, \cite{tipping1999} confirmed that this generative model is equivalent to PCA in the sense that the principal components of $\mathbf{X}$ can be retrieved using the maximum likelihood (ML) estimator $\mathbf{W}_{\textup{ML}}$ of $\mathbf{W}$. Indeed, if $\mathbf{A}$ is the $p \times d$ matrix of ordered principal eigenvectors of $\mathbf{X}^T\mathbf{X}$ and if $\boldsymbol{\Lambda}$ is the $d \times d$ diagonal matrix with corresponding eigenvalues, we have
\begin{equation} \label{ML}
\mathbf{W}_{\textup{ML}} = \mathbf{A}(\boldsymbol{\Lambda}-\sigma^2\mathbf{I}_d)^{1/2}\mathbf{R},
\end{equation}
where $\mathbf{R}$ is an arbitrary orthogonal matrix.

Several Bayesian treatments of this model have been conducted by using different priors on the loading matrix. However, the marginal likelihood of these models appeared to be untractable. To tackle this issue, several computational techniques were considered. The automatic relevance determination (ARD) prior was used together with Laplace \citep{bishop1999} or variational \citep{bishopvar,archambeau2009} approximations. \cite{minka2000} introduced more complex conjugate priors to perform Bayesian model selection on the dimension $d$ of the latent space using the Laplace approximation. Combined with variational inference, several sparsity inducing priors such as the Laplace \citep{guan2009}, the generalized hyperbolic \citep{archambeau2009} or the spike-and-slab \citep{lazaro2011} prior were also chosen for $\mathbf{W}$.

In this work, we aim at avoiding these approximations. Our approach is to investigate in which cases the marginal likelihood can be analytically computed. To this end, we will use the fact that, within the PPCA model~\eqref{modelePPCA}, the limit noiseless setting $\sigma\rightarrow0$ also allows to recover the principal components. This convenient framework  was first studied by \cite{roweis1998}  and has proven to be useful in several situations. The noiseless PPCA model was used for instance to facilitate inference in the presence of missing data \citep{yu2010,ilin2010}. More importantly in our context, it was successfully used by \cite{sigg2008} to enforce sparsity within an $\ell_1$ penalized PPCA framework -- which means that getting rid of the noise term is likely to be compatible with variable selection.

\subsection{A general framework for globally sparse PPCA}

In a classical (locally) sparse PCA context, the loading matrix $\mathbf{W}$ would be expected to contain few nonzero coefficients. However, to reach global sparsity, \emph{several entire rows} of  $\mathbf{W}$ have to be further constrained to be null. In this work, we handle variable selection using a binary vector $\mathbf{v} \in \{0, 1\}^p$ whose nonzero entries correspond to relevant variables. For technical purposes, we also denote $\bar{\mathbf{v}}$ the binary vector of $\{0,1\}^p$ whose support is exactly the complement of $\textup{Supp}(\mathbf{v})$. We denote $q=||\mathbf{v}||_0$ the number of relevant variables. In the PPCA framework, this leads to the following model for each observation
\begin{equation} \label{modelideal}
\mathbf{x} = \mathbf{V} \mathbf{W} \mathbf{y} + \boldsymbol{\varepsilon},
\end{equation} where $ \mathbf{V} = \diag ( \mathbf{v})$. Notice that the rows of $\mathbf{V} \mathbf{W}$,  corresponding to the zero entries of $\mathbf{v}$, are null. Therefore, the principal subspace will be generated by a basis of vectors which shares the sparsity pattern of $\mathbf{v}$. Such spaces spanned by a family of vectors sharing the same sparsity pattern will be called \emph{globally sparse subspaces}. This definition of global sparsity is closely related to the notion of \emph{row sparsity} introduced by \cite{vu2013}.

We further assume that the coefficients of the matrix $\mathbf{W}$ are endowed with the Gaussian priors $w_{ij}\sim \mathcal{N}(0,1/\alpha^2)$, for all $i,j$.
Following the empirical Bayes framework leads to seeking the parameters $\mathbf{v}$, $\alpha$ and $\sigma$ that maximizes the \emph{marginal likelihood} or \emph{evidence} 
\begin{align*}
p(\mathbf{X}|\mathbf{v},\alpha,\sigma) =\prod_{i=1}^n p(\mathbf{x}_i|\mathbf{v},\alpha,\sigma)   = \prod_{i=1}^n\int_{\mathbb{R}^{p \times d}} p(\mathbf{x}_i|\mathbf{W},\mathbf{v},\alpha,\sigma)p(\mathbf{W})d\mathbf{W}.
\end{align*} In previous Bayesian PCA models, the marginal likelihood was never derived because it was too difficult to compute in practice or even intractable. Here, conversely, the evidence of the model can be expressed analytically as a univariate integral using the isotropy of the prior on$\mathbf{W}$. In the following, $\mathbf{x}_\mathbf{v}$ denotes the subvector of $\mathbf{x}$ where only the columns corresponding to the nonzero indexes of $\mathbf{v}$ have been kept. Given a real order $\nu$, we denote by $J_\nu$ and $K_\nu$ the Bessel function of the first kind \cite[chap. 10 and 11]{AS}.
\begin{theorem} The density of $\mathbf{x}$ is given by \begin{multline} \label{exactTB}
p(\mathbf{x}|\mathbf{v},\alpha,\sigma)=e^{-\frac{||\mathbf{x}_\mathbf{\bar{v}}||_2^2}{2\sigma^2}} \sigma^{q-p} (2\pi)^{-p/2} ||\mathbf{x}_\mathbf{v}||_2^{1-q/2} \int_0^{\infty}\frac{u^{q/2}e^{-\sigma^2 u^2}}{(1+(u /\alpha)^2)^{d/2}}J_{q/2-1}(u||\mathbf{x}_\mathbf{v}||_2)du.
\end{multline}
\end{theorem}
A proof of this theorem is given in Appendix A. While reducing the dimension of the integration domain to one appears to be a valuable improvement, the integral of Equation~\eqref{exactTB}, albeit univariate, falls within the category of Hankel-like integrals known to be particularly delicate to compute. This is due to the fact that the integrand has singularities near the real axis \citep{ogata2005}. To overcome this limitation, we investigate in the following subsection the use of the noiseless PPCA model to obtain a tractable expression.

\subsection{A closed-form evidence for globally sparse noiseless PPCA}
\label{closedform}
To obtain a closed-form expression of the marginal likelihood, we consider the following modification of Model~\eqref{modelideal}. For the relevant variables, we use the noiseless PPCA model, and we assume that the irrelevant variables are generated by a Gaussian white noise. More specifically, we write
\begin{equation}
 \mathbf{x} = \mathbf{V} \mathbf{W} \mathbf{y} + \mathbf{\bar V} \boldsymbol{\varepsilon_1}+  \mathbf{ V} \boldsymbol{\varepsilon_2},
\end{equation}
where $\boldsymbol{\varepsilon_1} \sim \mathcal{N}(0, \sigma_1^2 \mathbf{I}_p)$ is the noise of the inactive variables and $\boldsymbol{\varepsilon_2} \sim \mathcal{N}(0, \sigma_2^2 \mathbf{I}_p)$ is the noise of the active variables, having in mind that we aim at investigating the noiseless limit $\sigma_2\rightarrow0$.  We will see that, with this particular formulation of the problem, the evidence has a closed form expression which involves the multivariate Bessel distribution, introduced by \citet[Def. 2.5]{fang}.

\begin{definition}
A random vector is said to have a \textbf{symmetric multivariate Bessel distribution} with parameters $\beta>0$ and $\nu>-k/2$ if its density is
$$\forall \mathbf{z} \in \mathbb{R}^k, \; \textup{Bessel}(\mathbf{z}|\beta,\nu)=\frac{2^{-k-\nu+1}\beta^{-k-\nu}}{\Gamma(\nu+k/2)\pi^{k/2}}||\mathbf{z}||_2^\nu K_\nu(||\mathbf{z}||_2/\beta).$$
\end{definition}

\begin{theorem}
In the noiseless limit $\sigma_2\rightarrow0$, $\mathbf{x}$ converges in probability to a random variable~$\tilde{\mathbf{x}}$ whose density is
\begin{equation} 
p(\tilde{\mathbf{x}}|\mathbf{v},\alpha,\sigma_1^2) = \mathcal{N}(\tilde{\mathbf{x}}_\mathbf{\bar v}|0,\sigma_1 \mathbf{I}_{p-q}) \textup{Bessel}(\tilde{\mathbf{x}}_\mathbf{v}|1/\alpha,(d-q)/2).
\label{noiseless}
\end{equation}
\end{theorem}
This theorem (proved in Appendix B) allows us to efficiently compute the noiseless marginal log-likelihood defined as $$ \mathcal{L}(\mathbf{X},\mathbf{v},\alpha,\sigma_1)= \sum_{i=1}^n \log \mathbb{P} (\tilde{\mathbf{x}}=\mathbf{x}_i|\mathbf{v},\alpha,\sigma_1).$$

Regarding hyper-parameter tuning, if we assume that $\mathbf{v}$ is known, the regularization parameter $\alpha$ can be optimized efficiently using univariate gradient ascent. In fact, as stated by next proposition (proved in Appendix C), the marginal log-likelihood is even a strictly concave function of $\alpha$.
\begin{proposition} The function $ \alpha \mapsto \mathcal{L}(\mathbf{X},\mathbf{v},\alpha,\sigma_1)$ is strictly concave on $\mathbb{R}^{*}_+$.
\end{proposition}
The unique optimal value $\hat{\alpha}$ can therefore be found easily using univariate convex programming.

The noise variance $\sigma_1$ can be estimated using \eqref{noiseless} by computing the standard error of the variables which were not selected by $\mathbf{v}$. However, since model~\eqref{modelideal} is a particular instance of PPCA, it is possible to use any regular PPCA noise variance estimator. A discussion on which estimator to choose is provided in subsection \ref{ss:comput}

\subsection{High-dimensional inference through a continuous relaxation}
\label{relaxation}

In spite of the results of the previous subsection, maximizing the evidence, even in the noiseless case, is particularly difficult (because of the discreteness of $\mathbf{v}$ which can take $2^p$ possible values). We therefore consider a simple continuous relaxation of the problem by replacing $\mathbf{v}$ by a continuous vector $\mathbf{u} \in [0,1]^p$. This relaxation is close to the one considered by~\cite{latouche} in a sparse linear regression framework. Denoting $\mathbf{U}=\diag (\mathbf{u})$, this relaxed model can be written as \begin{equation}\label{relaxedmodel} \mathbf{x} = \mathbf{U} \mathbf{W} \mathbf{y} + \boldsymbol{\varepsilon}.\end{equation}

We denote $\boldsymbol{\theta}=(\mathbf{u},\alpha,\sigma)$ the vector of parameters.
In order to maximize the evidence $p(\mathbf{X}|\boldsymbol{\theta})$, we adopt a variational approach  \citep[chap. 10]{bishop2006}. We view $\mathbf{y}_1,...\mathbf{y}_n$ and$\mathbf{W}$ as latent variables.

Given a (variational) distribution $q$ over the space of latent variables, the variational free energy is given by
\begin{equation}
\mathcal{F}_q(\mathbf{x_1},...\mathbf{x_n}|\boldsymbol{\theta})=-\mathbb{E}_q[\ln p(\mathbf{X},\mathbf{Y},\mathbf{W}|\boldsymbol{\theta})]-H(q),
\end{equation}
 where $H$ denotes the differential entropy, and is an upper bound to the negative log-evidence
$$-\ln p(\mathbf{X}|\boldsymbol{\theta})=\mathcal{F}_q(\mathbf{X}|\boldsymbol{\theta})-\KL(q||p(\cdot|\boldsymbol{\theta}))\leq \mathcal{F}_q(\mathbf{X}|\boldsymbol{\theta}).$$
To minimize $\mathcal{F}_q(\mathbf{X}|\boldsymbol{\theta})$, the following mean-field approximation is made on the variational distribution \begin{equation} q(\mathbf{Y},\mathbf{W})=q(\mathbf{Y})q(\mathbf{W}).
\end{equation}

With this factorization, a variational expectation-maximization (VEM) algorithm can be derived.
For the E-step, the variational posterior distribution $q^*$, which minimizes the free energy, is computed.

\begin{proposition}
\label{prop:estep}
The variational posterior distribution of the latent variables which minimizes the free energy is given by
\begin{equation}
q^*(\mathbf{Y})=\prod_{i=1}^n\mathcal{N}(\mathbf{y}_i|\boldsymbol{\mu}_i,\mathbf{\Sigma}),
\end{equation}
and
\begin{equation}
q^*(\mathbf{W})=\prod_{k=1}^p\mathcal{N}(\mathbf{w}_k|\mathbf{m}_k,\mathbf{S}_k),
\end{equation}
where, for all $i \in\{1,...,n\}$ and $k \in \{1,...,p\}$ $$ \boldsymbol{\mu}_i=\frac{1}{\sigma^2}\mathbf{\Sigma} \mathbf{M}^T\mathbf{U} \mathbf{x}_i\textup{, } \mathbf{m}_k=\frac{u_k}{\sigma^2}\mathbf{S}_k\sum_{i=1}^nx_{i,k}\boldsymbol{\mu}_i,$$
$$\mathbf{\Sigma}^{-1}=\mathbf{I}_d+\frac{1}{\sigma^2}\mathbf{M}^T\mathbf{U}^2\mathbf{M}+\frac{1}{\sigma^2}\sum_{k=1}^pu_k^2\mathbf{S}_k,\; \mathbf{S}_k^{-1}=\alpha^2\mathbf{I}_d + \frac{nu_k^2}{\sigma^2}\mathbf{\Sigma}+\frac{u_k^2}{\sigma^2}\boldsymbol{\mathcal{M}}^T\boldsymbol{\mathcal{M}},$$ $$\mathbf{M}=(\mathbf{m}_1,...\mathbf{m}_p)^T \; \; { and } \; \; \boldsymbol{\mathcal{M}}=(\boldsymbol{\mu}_1,...\boldsymbol{\mu}_n)^T.$$
\end{proposition}
It is worth noticing that two factorizations arise naturally. The four equations of Proposition \eqref{prop:estep} (proved in Appendix D) will constitute the E-step of the VEM algorithm used to minimized the free energy.

We can now compute the negative free energy which will be maximized during the M-step.

\begin{proposition}
Up to unnecessary additive constants, the negative free energy is given by
\begin{multline}-\mathcal{F}_q(\mathbf{x_1},...\mathbf{x_n}|\boldsymbol{\theta})=-np\ln \sigma + dp \ln \alpha - \frac{1}{2\sigma^2}\Tr (\mathbf{X}^T\mathbf{X})- \frac{1}{2\sigma^2}\sum_{k=1}^pu_k^2\Tr[(n \boldsymbol{\Sigma}+\boldsymbol{\mathcal{M}}^T\boldsymbol{\mathcal{M}})(\boldsymbol{S}_k+\mathbf{m}_k\mathbf{m}_k^T)]\\+ \frac{1}{\sigma^2}\sum_{i=1}^n\mathbf{x}_i^T\mathbf{UM}\boldsymbol{\mu}_i   + \sum_{k=1}^p -\frac{\alpha^2}{2}\Tr(\mathbf{S}_k+\mathbf{m}_k\mathbf{m}_k^T)-\frac{1}{2}\sum_{i=1}^n \Tr(\boldsymbol{\Sigma}+\boldsymbol{\mu}_i\boldsymbol{\mu}_i^T) +\frac{n}{2}\ln |\mathbf{\Sigma}|+\frac{1}{2}\sum_{k=1}^p\ln |\mathbf{S}_k|. 
\end{multline}
\end{proposition}

Minimizing the free energy leads to the following M-step updates
\begin{equation} \label{eq:Malpha}
\alpha^*=\left(\frac{1}{dp}\sum_{k=1}^p\Tr(\mathbf{S}_k+\mathbf{m}_k\mathbf{m}_k^T)\right)^{-1/2},
\end{equation}
\begin{equation} \label{eq:Msigma}
\sigma^*=\sqrt{\frac{\Tr (\mathbf{X}\mathbf{X}^T+ \mathbf{X}\mathbf{U}\mathbf{M}\boldsymbol{\mathcal{M}})}{np} +\frac{1}{np}\sum_{i=1}^n\sum_{k=1}^pu_k^2\Tr[(\boldsymbol{\Sigma}+\boldsymbol{\mu}_i\boldsymbol{\mu}_i^T)(\boldsymbol{S}_k+\mathbf{m}_i\mathbf{m}_i^T)]},
\end{equation} and, for $k \in \{1,...,p\}$,
\begin{equation} \label{eq:Mu}
u_k^*=\argmin_{u \in [0,1]} \frac{u^2}{2\sigma^2} \sum_{i=1}^n\Tr[(\boldsymbol{\Sigma}+\boldsymbol{\mu}_i\boldsymbol{\mu}_i^T)(\boldsymbol{S}_k+\mathbf{m}_i\mathbf{m}_i^T)]-u\sum_{i=1}^nx_{i,k}\mathbf{m}_k^T\boldsymbol{\mu}_i.
\end{equation}
Note that the objective function of the optimization problem~\eqref{eq:Mu} is simply a univariate polynomial. 

\subsection{The GSPPCA algorithm}

Once the VEM algorithm has converged, the continuous vector $\mathbf{u}$ still needs to be transformed into a binary one. To do so, the following simple procedure, summarized in Algorithm~\ref{algofinal}, is considered: \begin{itemize} \item a family of $p$ nested models is built using the order of the coefficients of $\mathbf{u}$ as a way of ranking the variables. Specifically, for each $k\leq p$, the $k$-th element of this family is the binary vector $\mathbf{v}^{(k)}$ such that the $k$ top coefficients of $\mathbf{u}$ are set to 1 and the others to 0.
\item the marginal likelihood $\mathcal{L}$ of the non-relaxed model (computed using the formula of Theorem 3) is then maximized over this family of models.
\item the model $\mathbf{v}$  with the largest marginal likelihood is kept.
\end{itemize}
Once the model is estimated, the globally sparse principal components of $\mathbf{X}$ can be computed by simply performing PCA on $\mathbf{X}_\mathbf{v}$. This type of post-processing is similar to the \emph{variational renormalization} introduced by \cite{mog2005}. In the case of local sparsity, variational renormalization can be achieved using an alternating maximization scheme \citep{journee2010}. However, the global sparsity structure greatly simplifies this procedure by reducing it to performing PCA on the relevant variables.

\begin{algorithm}[t] \label{algofinal}
\caption{GSPPCA algorithm for unsupervised variable selection}         

\KwIn{data matrix $\mathbf{X} \in \mathbb{R}^{n\times p}$, dimension of the latent space $d \in \mathbb{N}^*$ }

\KwOut{sparsity pattern $\mathbf{v} \in \{0,1\}^p$}

\BlankLine

\BlankLine

// VEM algorithm to infer the path of models \\
Initialize $\mathbf{u},\alpha,\sigma,\boldsymbol{\mu}_1,...,\boldsymbol{\mu}_n,\mathbf{m}_1,...,\mathbf{m}_p,\mathbf{S}_1,...,\mathbf{S}_p$ and $\boldsymbol{\Sigma}$    \;

   \Repeat{convergence of the variational free energy}{
   E-step from Proposition \ref{prop:estep}\;
   M-step from equations \eqref{eq:Malpha},\eqref{eq:Msigma},\eqref{eq:Mu}\;
 }
  
  \BlankLine
  \BlankLine

   // Model selection using the {exact} marginal likelihood   \\
  Compute $\sigma_1$ \;
  \For{k = 1..p}{
  Compute $\mathbf{v}^{(k)}$\;
  Find $ \alpha_k = \argmax_{\alpha>0} \{ \alpha \mapsto \mathcal{L}(\mathbf{X},\mathbf{v}^{(k)},\alpha,\sigma_1) \}$ using gradient ascent \;
  }
  
  \BlankLine
\BlankLine

  $ q=\argmax_{1\leq k\leq p} \mathcal{L}(\mathbf{X},\mathbf{v}^{(k)},\alpha_k,\sigma_1)$ \;
   $\mathbf{v}=\mathbf{v}^{(q)}$ \;
   
\label{algo}

\end{algorithm}

\subsection{Links with other sparsity-inducing Bayesian procedures}

\paragraph{Spike-and-slab models} Model \eqref{modelideal} may be rewritten $\mathbf{x}=\tilde{\mathbf{W}}\mathbf{y}+\boldsymbol{\varepsilon}$ where $\tilde{\mathbf{W}}=\mathbf{VW}$. The prior distribution for the parameter $\tilde{\mathbf{W}}$ is similar to the spike-and-slab prior introduced by \cite{mitchell88} in a linear regression framework. Indeed, each coefficient $\tilde w_{ij}$  follows \emph{a priori} either a Dirac distribution with mass at zero (if $v_i=0$) which is usually called the \emph{spike} or a Gaussian distribution with variance $1/\alpha^2$ (if $v_i=1$) which is usually called the \emph{slab}. However,  contrary  to  standard spike-and-slab  models which would assume  a product of  Bernoulli prior distributions over  $\mathbf{v}$,  we  see $\mathbf{v}$ here as a deterministic parameter to be inferred from the data. It is worth noticing that spike-and-slab priors have already been applied to locally sparse PCA by \cite{lazaro2011} and \cite{mohamed2012}.

\paragraph{Automatic relevance determination} Introduced in the context of feedforward neural networks \citep{mackay1994,neal1996}, automatic relevance determination (ARD) is a popular empirical Bayes procedure to induce sparsity. ARD was applied to Bayesian PCA models together with VEM algorithms in order to obtain automatic dimensionality selection \citep{bishopvar} of local sparsity \citep{archambeau2009}. In order to obtain global sparsity, ARD may be built using Model~\eqref{modelePPCA} together with Gaussian priors $\mathbf{w}_i\sim \mathcal{N}(0,a_i\mathbf{I}_d)$ for  $i \in \{1,...,p\}$. Similarly to \cite{tipping2001}, maximizing the marginal likelihood would discard irrelevant variables by leading several variance parameters $a_i$ to vanish. Interestingly, this model is somehow related to the relaxed GSPPCA model. Indeed the relaxed model~\eqref{relaxedmodel} assumes that the $i$-th line of the loading matrix $\mathbf{UW}$ follows \emph{a priori} a  $\mathcal{N}(0,u_i^2/\alpha^2\mathbf{I}_d)$ distribution.
The relaxed model will consequently inherit the good properties of ARD -- listed for example by \cite{wipf2011}. However, similarly to \cite{latouche}, using the exact marginal likelihood to eventually obtain a sparse solution will avoid many classical drawbacks of ARD. First, as pointed out by~\cite{wipf2008}, convergences of EM algorithms are extremely slow in the case of the ARD models. However, with our approach, since we only need the \textit{ordering} of the coefficients of $\mathbf{u}$, we do not have to wait for the complete convergence of this parameter. In practice, in all the experiments that we carried out, we only had to perform less than a few hundreds of iterations of the algorithm to obtain convergence of the free energy in order to perform variable selection. It is worth mentioning that the fact that the objective function converges faster than the parameters of the model is a quite general property of EM algorithms \citep{xu1996}. Our procedure also avoids the lack of flexibility of ARD by computing posterior probabilities of models rather than simply giving an estimate of the best sparse model. {Combined with a greedy technique similar to Occam's window \citep{madigan1994}}, this feature could allow for example to perform Bayesian model averaging, which is not possible with ARD. Eventually, in the context of Bayesian PCA, ARD models such as the ones of \cite{bishop1999,bishopvar} or \cite{archambeau2009} have to rely on approximations of the marginal likelihood while we use an exact expression.

\subsection{Computational considerations}
\label{ss:comput}
\paragraph{Intrinsic dimension estimation} Since model \eqref{modelideal} is a particular instance of PPCA, any intrinsic dimension estimator for PCA can be applied to estimate beforehand the intrinsic dimension $d$. Although the problem of finding $d$ is of critical importance, we assume in this work that a reasonable choice of dimension has already been made by the practitioner. While it could be tempting to use the exact noiseless marginal likelihood to select $d$, the close relationship existing between the noise level and $d$ in PPCA \citep{tipping1999,nakajima2011} suggests that loosing the noise information is likely to be prejudicial for intrinsic dimension estimation.
{
\paragraph{Initialization strategies for the VEM algorithm} Regarding the initialization of the relaxed model parameter $\mathbf{u}$, we chose to initialize all its coefficients to one. This allows to avoid premature vanishing of these coefficients which is a common drawback of ARD-like techniques \citep{wipf2008}. The noise standard error can be simply initialized using any classical PPCA noise estimator (cf. subsection \ref{closedform}). Similarly to \cite{latouche}, the slab precision parameter $\alpha$ controls the sparsity of the VEM solution and a too small initial value is likely to lead to a too sparse solution such as the useless local optimum $\mathbf{u} = 0$. Following \cite{biernacki2003}, we chose to perform short VEM runs (with less than 5 iterations) on a small grid (typically {$\alpha \in \{0.1,1,10\}$}) and to select the value of $\alpha$ that led to the lowest free energy. The posterior means of the PCA loadings $\mathbf{m}_1,...,\mathbf{m}_p$ and of the corresponding scores $\boldsymbol{\mu}_1,...,\boldsymbol{\mu}_n$ can be initialized using the singular vectors of $\mathbf{X}$. If the size of the data forbids to perform this SVD, using random standard Gaussian coefficients as starting points does not significantly alter the results. Finally, the initial values chosen for the posterior covariance matrices are $\boldsymbol{\Sigma} = \mathbf{I}_d$ and $\mathbf{S}_1=...=\mathbf{S}_p=\alpha^{-2}\mathbf{I}_d$.}
\paragraph{Computational cost of VEM iterations}
Thanks to the factorizations that arised naturally during variational inference, the cost of each VEM iteration is of order $O(pnd^3)$ which is linear \emph{both in sample size and dimensionality} and therefore particularly suitable for high-dimensional inference.
\paragraph{Estimation of the noise variance} As mentioned in seubsection \ref{closedform}, the standard error $\sigma_1$ of irrelevant predictors can be estimated using any regular PPCA estimator. Specifically, three important estimators are considered: the maximum likelihood estimator \citep{tipping1999}, its unbiased correction \citep{passemier2015}, or simply the median of the variances of all features \citep{johnstone2009}. Since the ML estimator is known to be biased in the high-dimensional regime, it is usually preferable to use its bias-corrected version. Both of these estimators can also be computed using the singular value decomposition (SVD) of $\mathbf{X}$. Note that since the median estimator does not need to perform this decomposition, it is therefore more suitable for large-scale inference.
\paragraph{Large scale inference} In the GSPPCA algorithm, SVD is used twice. Indeed, the top $d$ singular vectors can be used to initiate the VEM algorithm and the $p-d$ smallest singular values can be used to estimate the noise variance (both as a VEM starting point for $\sigma$ and as an estimator for $\sigma_1$). This can be done efficiently using a \emph{truncated SVD algorithm}. We chose specifically the R interface \citep{rspectra} of the Spectra\footnote{http://yixuan.cos.name/spectra/index.html} C++ library. However, for very large scale problems, even a fast truncated SVD algorithm appears computationally prohibitive. To tackle this issue, we offer two alternatives. First, the covariance matrices initialized using the eigenvectors can be initialized using random standard Gaussian coefficients. Moreover, following \cite{johnstone2009}, the noise variance can be estimated using the median of the variable variances.
This leads to a "SVD-free" version of the GSPPCA algorithm suitable for very large scale problems.
{
\paragraph{Model selection speedup} The model selection step of the GSPPCA algorithm requires to perform $p$ univariate gradient ascents, which can be computationally expensive when $p$ is large. A simple way to reduce the number of gradient ascents is to rely on the links between our relaxed model and ARD. Specifically, we can discard \emph{before} the model selection step all the variables corresponding to the subset $\{i \in \{1,...,p\} | u_i=0 \}$ where $\mathbf{u}$ is the relaxed model parameter obtained after convergence of the VEM algorithm. When $\mathbf{u}$ is sparse, this will bring about a substantial speedup. Notice that, since ARD is known to converge slowly, $\mathbf{u}$ is unlikely to be sparse enough and the model selection step is still necessary.}
\paragraph{Evaluation of Bessel functions}
The modified Bessel function of the second kind, which is used to compute the exact marginal likelihood and it gradient with respect to $\alpha$, can be delicate to compute as soon as its order or its argument is large. In our experiments, we tackled this issue by using an asymptotic expansion based on Debye polynomials \cite[formula 9.8.7]{AS}. This is in particular implemented in the R package \texttt{Bessel} \citep{maechler}. We found this approximation to be extremely accurate in all the experiments that we carried out.

\section{Numerical simulations}

This section aims at highlighting the specific features and abilities of the proposed GSPPCA approach on simulated and real data sets.

\subsection{An introductory example}
We consider here a simple introductory example to illustrate the proposed combination between a relaxed VEM algorithm and the closed-form expression of the marginal likelihood. For this experiment, $n=50$ observations are simulated according to \eqref{modelideal} with $p=30$, $d=5$ and $q=10$. Each coefficient of $\mathbf{W}$ is drawn at random according to a standard Gaussian distribution and the noise variance is equal to $0.1$. Figure~1 presents the results of GSPPCA on this toy data set. The left panel presents in dark blue the coefficients of the estimated $\mathbf{u}$ obtained after running the VEM algorithm (sorted in decreasing order) and the corresponding true values of $\mathbf{v}$ (pale blue points) used in the simulations. The right panel shows the values of evidence computed on the family of models inferred by the order of the coefficients of $\mathbf{u}$. On this simple example, $\mathbf{u}$ captures the true ranking of the variables and the model with the largest evidence is actually the true one.

\begin{figure}[t]  \caption{Variable selection with GSPPCA on the introductory example.}
\centering

   \includegraphics[width=0.95\columnwidth]{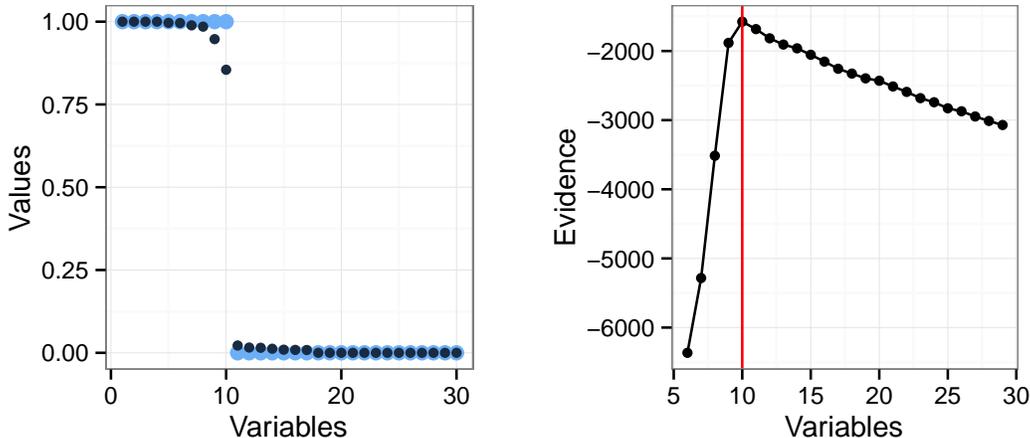}
\end{figure}

\subsection{Range of the noiseless assumption}
\label{subsec:range}
In all the experiments that we carried out, since the noiseless PPCA model is not a true generative $p$-dimensional model (the random variable $\tilde{\mathbf{x}}$ belongs to a strict subspace of $\mathbb{R}^p$), we chose not to use it to generate data in our experiments. We rather chose the more realistic and natural Model~\eqref{modelideal}. Since this model includes a nonzero noise, it is important to know the limits of the noiseless assumption.

We therefore simulated two scenarios according to Model~\eqref{modelideal}: {a first one with $n=40$ observations and a second one with $n=200$. In both scenarios, $p=200$}, $d=10$, $q=20$, and each coefficient of $\mathbf{W}$ is drawn  according to a standard Gaussian distribution. The sparsity pattern chosen is simply
\begin{equation} \label{eq:pattern} \mathbf{v}=(\overbrace{1,...,1}^{\textup{20 times}},\overbrace{0,...,0}^{\textup{180 times}})^T. \end{equation}
In this simple simulation scheme, the signal-to-noise ratio (SNR) may be defined as $\label{SNR} \textup{SNR}=\frac{1}{p\sigma^2}\mathbb{E}_\mathbf{W}[(\mathbf{VW})^T\mathbf{VW}]p\sigma^2=\frac{dq}{p\sigma^2}.$
We chose a linear grid of $20$ SNR ranging from {$0.1$ (most difficult scenario) to $3$} (easiest scenario) and generated $100$ datasets for each noise level. To evaluate the quality of the variable selection, we computed the F-score between $\hat{\mathbf{v}}$ and $\mathbf{v}$ on 100 runs. We recall that the F-score is the harmonic mean of precision and recall, and is closer to $1$ when the selection is faithful. Unsurprisingly, when the SNR gets close to zero, the quality of the variable selection diminishes. However, GSPPCA appears to be quite robust to noise, even though the data are not generated according to the underlying noiseless model. {Indeed, even in the case where $n=40$, we observe an almost perfect recovery as long as SNR>0.5.}

\begin{figure}[t]\caption{Median, first and third quartiles of the F-score for the experiment of subsection 3.2, based on 100 runs}
\centering
  \includegraphics[width=0.95\columnwidth]{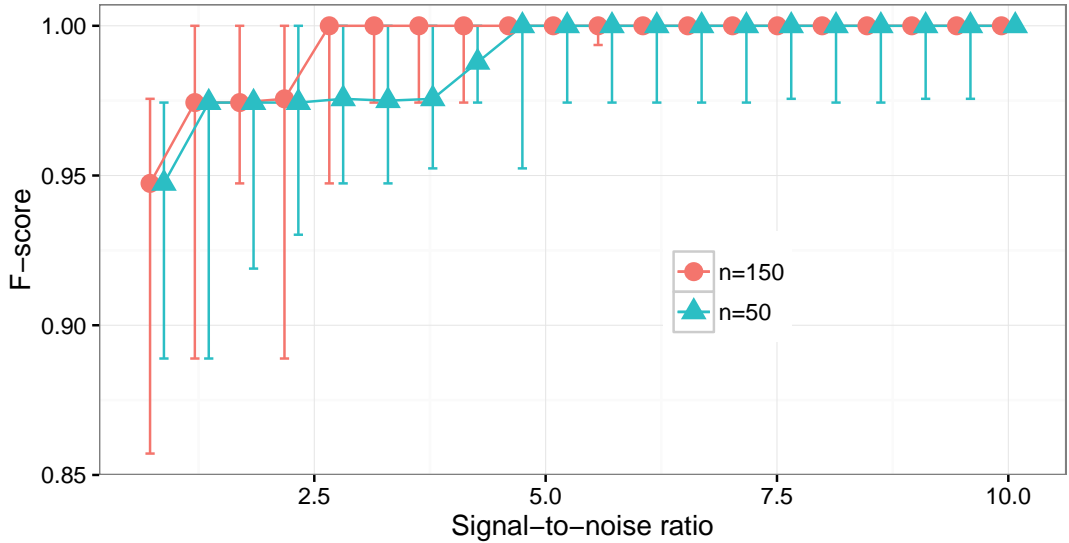}
\end{figure}

\subsection{Model selection}
\label{sub:modelselection}
In this subsection, we compare the model selection accuracies of two global methods -- GSPPCA, SSPCA \citep{jenatton2009} --  and a local one -- SPCA  \citep{zou2006}.

\paragraph{Simulation setup}
While the simple simulation setup of Subsection \ref{subsec:range} conveniently allowed to compute the SNR in closed formed in order to assess the range of the noiseless assumption, we introduce here a more realistic scheme by considering a finer correlation structure as well as a non-Gaussian noise. Specifically, first we generate $n$ i.i.d observations $(\mathbf{z}_1,...,\mathbf{z}_n)$ following multivariate normal distribution $\mathcal{N}(0,\mathbf{R})$ where $\mathbf{R}=\diag(\mathbf{R}_1,...,\mathbf{R}_4)$ is a 4-blocks diagonal matrix where $R_\ell$ is such that $r_{\ell ii} = 0.3$ and  $r_{\ell ij} = \rho$ for $i,j=1,\dots,p/4$ and $i\neq j$.
Then, a globally sparse PCA model is obtained as followed. First, PPCA is performed on the sample $(\mathbf{z}_1,...,\mathbf{z}_n)$, which leads to a non-sparse ML estimate  $\mathbf{W}_{\textup{ML}}$ for the loading matrix. Then, given a sparsity pattern $\mathbf{v} \in \{0,1\}^p$ and denoting $\mathbf{V}=\diag(\mathbf{v})$ as before, the loading matrix matrix is "globally sparsified" by considering $\mathbf{VW}_{\textup{ML}}$. The final observations are eventually generated according to the non-noiseless model
\begin{equation}
\label{eq:simu}
\forall i \leq n, \; \; \mathbf{x}_i = \mathbf{V} \mathbf{W}_{\textup{ML}} \mathbf{y}_i + \boldsymbol{\varepsilon}.
\end{equation}
The simple sparsity pattern \eqref{eq:pattern} is kept and the vectors $\mathbf{y_1},...,\mathbf{y}_n$ are standard Gaussian as in regular PPCA. Regarding the noise term $\boldsymbol{\varepsilon}$, we consider two scenarios. A first one with Gaussian noise and a second one with Laplacian noise, both centered with unit variance.
We choose $p=200$, $d=10$, $q=20$ and consider five cases for the sample size: $n=p/5$, $p/4$, $n=\lfloor p/3\rfloor$, $n=p/2$ and $n=p$. More classical $n>p$ cases are not presented here since regular PCA is known to perform well in this context and variable selection thus may not be of great use \citep{johnstone2009}. Each experiment was repeated 50 times.

\paragraph{Model selection criteria} Regarding SSPCA, we used the Matlab code available at the main author's webpage and chose the tuning parameter using $5$-fold cross-validation on the reconstruction error. We constrained the algorithm in order to obtain globally sparse solutions. For SPCA, we used the \texttt{elasticnet} R package and an \emph{ad-hoc} method by selecting enough variables to explain $99\%$ of the total variance. We also tried to apply another globally sparse algorithm, vsnPCA-$\ell_0$ from \cite{ulfarsson2011}. However, their use of the Bayesian information criterion (BIC) led to selecting very few variables. This is not very surprising: since BIC is an asymptotic sparsity criterion, it is thus likely to perform poorly when $p$ is larger than $n$.

\paragraph{Results}
Tables \ref{tab:gauss} and \ref{tab:laplace} reports the mean and standard error of the F-score for the experiments described is this subsection. The two globally sparse methods vastly outperform SPCA, which is unable to identify the particular structure of the data. {When $p$ is larger than $n/2$, both globally sparse algorithms perform very well, GSPPCA being slightly better in the Gaussian noise case. It is not surprising to see SSPCA adapt efficiently to Laplacian noise because cross-validation is a model-free technique and is more likely to outperform model-based techniques when the data is not generated according to the model distribution.
However, when $n$ is smaller than $p/2$, GSPPCA significantly outperforms SSPCA in both noise scenarios.} This reminds the fact that, is many $p\gg n$ situations, Bayesian model selection empirically outperforms $\ell_1$-based methods \citep{celeux2012,latouche}.
\begin{table}[t!]
\centering
\caption{F-score$\times100$ for the model selection experiment of subsection \ref{sub:modelselection} with Gaussian noise}
\label{tab:gauss}
\begin{tabular}{l|ccccc}
      &$n=p/5$   & $n=p/4$   &    $n=\lfloor p/3\rfloor$   & $n=p/2$           & $n=p$  \\ \hline
SPCA   &$20.7\pm0.7$ & $21.2\pm0.7$  & $21.5\pm0.7$ & $21.7\pm0.5$  & $25.2\pm2.1$  \\ \hline
SSPCA &$66.7\pm21.4$ & $71.5\pm20$ & $86.7\pm14.2$ & $ 95.6\pm8.9$ & $ 98.2\pm7.2$  \\ \hline
GSPPCA   & $\bf 86.8\pm7.06$ & $ \bf 93.9\pm3.66$  & $ \bf97.2\pm2.55$ & $\bf 99.2\pm1.4$  & $\bf 1\pm0$ 
\end{tabular}
\end{table}
\begin{table}[t!]
\centering
\caption{F-score$\times 100$ for the model selection experiment of subsection \ref{sub:modelselection} with Laplacian noise}
\label{tab:laplace}
\begin{tabular}{l|ccccc}
       &$n=p/5$  & $n=p/4$   & $n=\lfloor p/3\rfloor$      & $n=p/2$           & $n=p$   \\ \hline
SPCA & $20.8\pm0.6$ & $21.3\pm0.6$  & $21.6\pm0.8$ & $21.8\pm0.6$  & $25.3\pm1.7$ \\ \hline
SSPCA &$60.6\pm22.4$ & $63.9\pm25.2$ & $82.7\pm18.1$  & $\bf 94.2\pm10.2$ & $ 97.4\pm9.5$  \\ \hline
GSPPCA  & $\bf 74.2\pm10$  & $\bf 77.6\pm9.09$  &$\bf 79.7\pm8.38$ & $ 88\pm5.95$  & $\bf 99.2\pm1.4$   
\end{tabular}
\end{table}

\subsection{Global versus local}

Here, we illustrate on real data sets how using GSPPCA instead of computing the leading sparse principal component for model selection can lead to selecting more relevant variables -- i.e variables that retain more variance or are more interpretable.

\paragraph{Explained variance} We consider the breast cancer data base from the \texttt{breastCancerVDX} R package \citep{VDX}, consisting  in expression levels of $p=5391$ genes for $n=344$ breast cancer patients. {More details regarding this data set -- including the preprocessing technique used -- are given in Appendix F.} Given a cardinality $q$, we applied three methods to select relevant genes: \begin{itemize}
\item we computed the first $q$-sparse principal component using SPCA \citep{zou2006}
\item we computed the support of the globally $q$-sparse subspace of dimension $d=10$ using GSPPCA and SSPCA
\end{itemize}
For each method, we projected the data onto a 10-dimensional globally $q$-sparse subspace using the sparsity pattern found by the algorithm and computed the percentage of explained variance using the criterion introduced by \cite{shen2008} -- for each method, we applied the post-processing technique of \cite{mog2005}. The results are plotted on Figure~\ref{fig:var}. It is important to notice that both global methods explain much more variance than SPCA. This fact is not surprising since the data is indeed projected onto a globally sparse subspace, but the significance of this variance gap highlights the fact that different dimensions lead to very different sparsity patterns. This means that projecting the data onto a single sparse axis is likely to lead to an important information loss (this fact is confirmed in section 5).
The variables selected by GSPPCA retain significantly more variance than the ones selected by SSPCA, and may consequently be of superior interest.
\begin{figure}[tb] \label{fig:var} \caption{Percentage of variance explained by the data projected onto a 10-dimensional globally sparse subspace}
\centering
   \includegraphics[width=0.95\columnwidth]{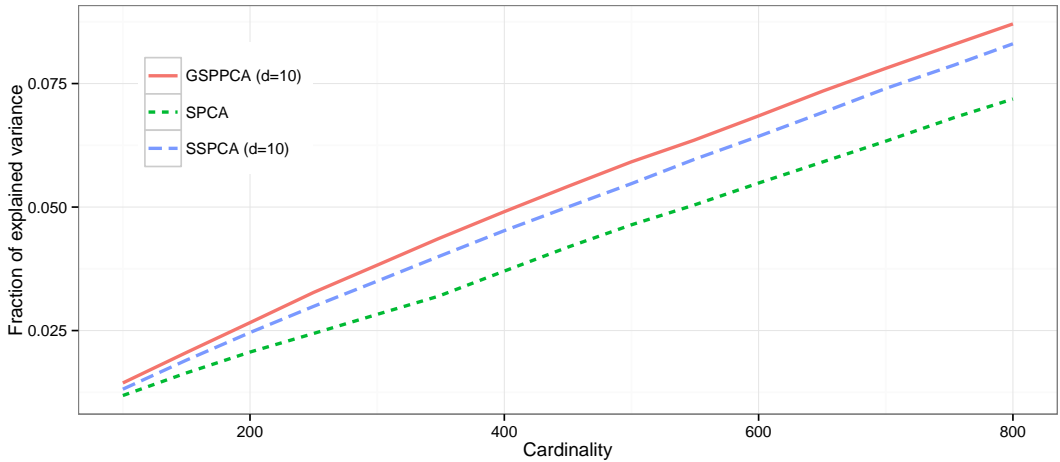}
\end{figure}

\paragraph{Interpretability} Inspired by \citet[section 8.2.3.1]{hastie2015}, we consider the problem of learning which features are relevant on three data sets of handwritten digits. We consider $n=500$ gray-scale images (with $p=758$ pixels) of handwritten sevens from three data sets introduced by \cite{larochelle2007}:
\begin{itemize}
\item \emph{mnist-basic} which is simply a subsample of sevens from the original MNIST data set,
\item \emph{mnist-back-rand} in which random backgrounds were inserted in the images. Each pixel value of the background was generated uniformly between 0 and 255,
\item \emph{mnist-back-image} in which random patches extracted
from a set of 20 grey-scale natural images were used as backgrounds for the
sevens.
\end{itemize}
On these three data sets, we apply SPCA (with $d=1$), SSPCA and GSPPCA (both with $d=100$) in order to select $q=200$ relevant pixels. On \emph{mnist-basic}, even if SPCA's result is a little bit more erratic than the two others, all selections are interpretable and we can easily recognize a seven. On \emph{mnist-back-rand} however, while the two globally sparse selections are still consistent, SPCA's pixels are more scattered and it is harder to recognize the shape of a seven. Eventually, on \emph{mnist-back-image}, GSPPCA's selection is less smooth but a seven can still be recognized, whereas SPCA appears to randomly select pixels \emph{almost everywhere but near the mean seven}. SSPCA seems to notice that the zone occupied by the upper bars of the sevens is of interest, but its selection does not appear interpretable.
\begin{table}[t]
\tiny
\centering
\caption{Variable selection of SPCA and GSPPCA for the three datasets of \cite{larochelle2007}, selected variables are in white}
\label{mnist}
\begin{tabular}{c c c c}
  & \emph{mnist-basic}                                                    & \emph{mnist-back-rand}                                                    & \emph{mnist-back-image}                                                    \\ \hline
Sample & \includegraphics[width=1.85cm,trim={5.9cm 2.6cm 4.9cm 2cm},clip=true]{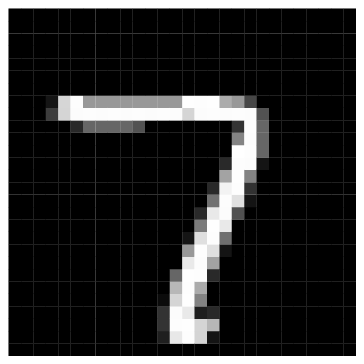} & \includegraphics[width=1.85cm,trim={5.9cm 2.6cm 4.9cm 2cm},clip=true]{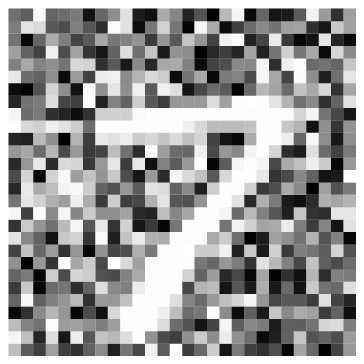} & \includegraphics[width=1.85cm,trim={5.9cm 2.6cm 4.9cm 2cm},clip=true]{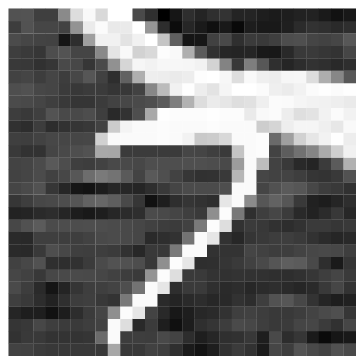} \\ \hline
SPCA & \includegraphics[width=1.85cm,trim={5.9cm 2.6cm 4.9cm 2cm},clip=true]{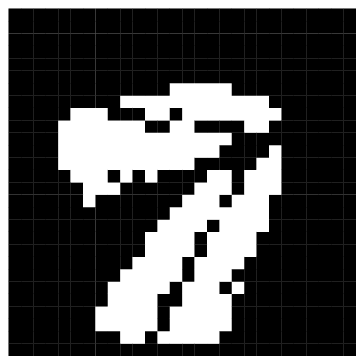} & \includegraphics[width=1.85cm,trim={5.9cm 2.6cm 4.9cm 2cm},clip=true]{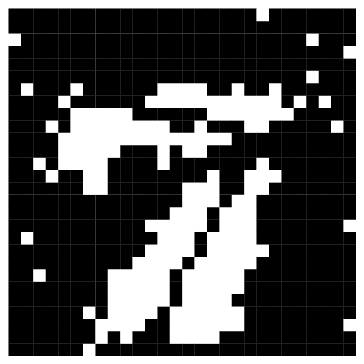} & \includegraphics[width=1.85cm,trim={5.9cm 2.6cm 4.9cm 2cm},clip=true]{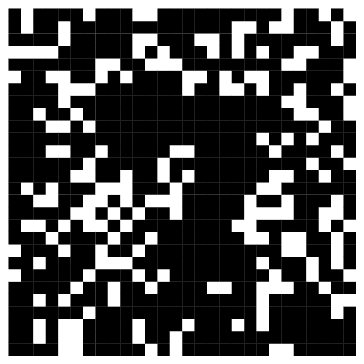} \\ \hline
SSPCA & \includegraphics[width=1.85cm,trim={5.9cm 2.6cm 4.9cm 2cm},clip=true]{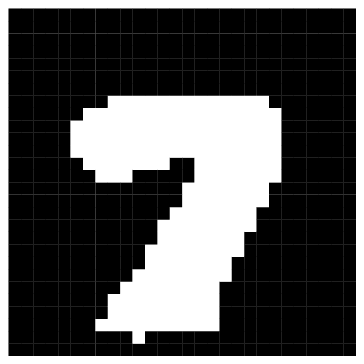} & \includegraphics[width=1.85cm,trim={5.9cm 2.6cm 4.9cm 2cm},clip=true]{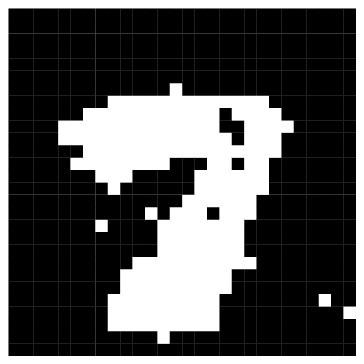} & \includegraphics[width=1.85cm,trim={5.9cm 2.6cm 4.9cm 2cm},clip=true]{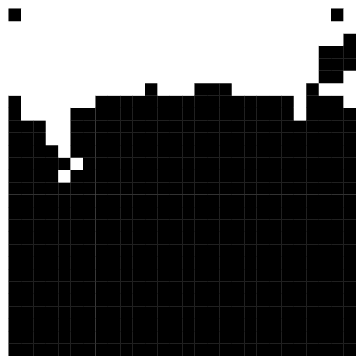} \\ \hline
GSPPCA & \includegraphics[width=1.85cm,trim={5.9cm 2.6cm 4.9cm 2cm},clip=true]{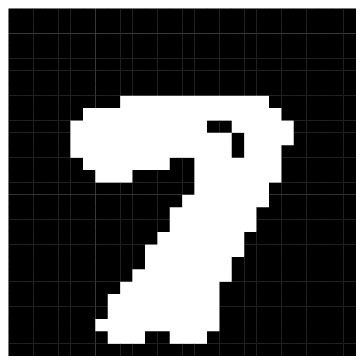} & \includegraphics[width=1.85cm,trim={5.9cm 2.6cm 4.9cm 2cm},clip=true]{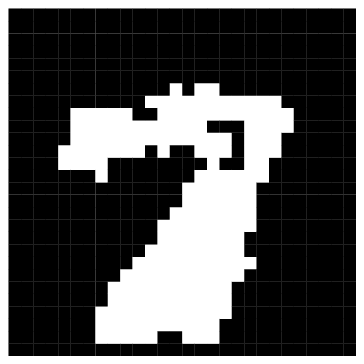} & \includegraphics[width=1.85cm,trim={5.9cm 2.6cm 4.9cm 2cm},clip=true]{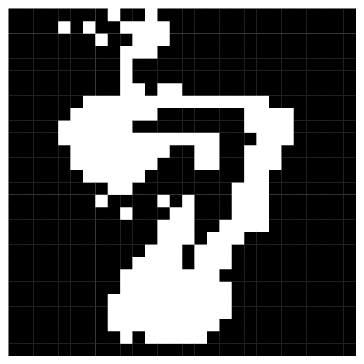} \\ 
\end{tabular}
\end{table}

\section{Application to signal denoising}

In this section, we focus on a first possible application of GSPPCA for signal
denoising through the sparsification of a wavelet decomposition. PCA
is indeed a popular way to denoise multivariate signals~\citep{aminghafari2006,johnstone2009}.
To illustrate the potential interest of GSPPCA in this context, we
consider hereafter two simulation scenarios, each using a specific
form of signal and wavelet. The simulation scenarios are as follows:
\begin{itemize}
\item Scenario A: it consists in a square wave signal with 6 states of different
lengths. The observed signal is sampled with a time step of $5\times10^{-3}$
with an additional Gaussian noise with zero mean and $0.2$ standard deviation. The Haar wavelet is used here for signal reconstruction.
\item Scenario B: the original signal is here a mixture of 4 Gaussian densities.
The observed signal is also sampled with a time step of $5\times10^{-3}$
with an additional Gaussian noise with zero mean and $0.2$ standard deviation. The Daubechies D8 wavelet is used here for signal reconstruction.
\end{itemize}
Figure~\ref{wav-2} presents the original signals and observed signals
for scenarios A and B.
In both cases, $n=100$ signals were sampled during the training phase and decomposed as $p=175$
wavelet coefficients. For signal denoising, GSPPCA is applied on the
$n\times p$ wavelet coefficient matrix to extract $d=10$ globally sparse
principal axes. Then, a new sampled signal is projected on those extracted
principal axes and back-projected in the original wavelet domain.
It is worth mentioning that the estimated value for $q=\left\Vert v\right\Vert _{0}$
is $17$ on scenario A and $15$ on scenario B. 

As an illustration, we plotted on Figure~\ref{wav-2} the denoising
results for newly sampled signals A and B with GSPPCA. We used the same projection-reconstruction protocol for PCA, thresholded PCA (PCA
loading smaller than $1\times10^{-3}$ are set to $0$) and SPCA ($\lambda$
is chosen such that $99\%$ of the PCA projected variance is conserved).
Denoising results obtained with those methods are also supplied on
Figure~\ref{wav-2}. First, on both signal A and
B, PCA achieves a very satisfying denoising and thus confirms his
validity in this context. One can also show that a simple thresholding
of the PCA loadings allows a clear denoising improvement and
turns out to be competitive with the one performed by SPCA. The SPCA
result is here somehow disappointing due to the fact that the sparsity
is not global and most wavelet levels stay active in the final reconstruction.
Finally, the global sparsity of GSPPCA retains only a few wavelet
levels and achieves here the best reconstruction in both scenarios.

Finally, Table~\ref{wav-4} presents the reconstruction error (sum
of squared errors) averaged on 50 test signal reconstructions, on the two
simulation scenarios. The results confirms the observations made on
Figure~\ref{wav-2}. GSPPCA achieves particularly
good performances on both scenarios and thus imposes itself as
a competitive tool for signal denoising. Moreover, the GSPPCA reconstruction uses fewer wavelet
levels and is therefore visually smoother.

\begin{figure}[p]
\begin{centering}
\includegraphics[width=0.87\columnwidth]{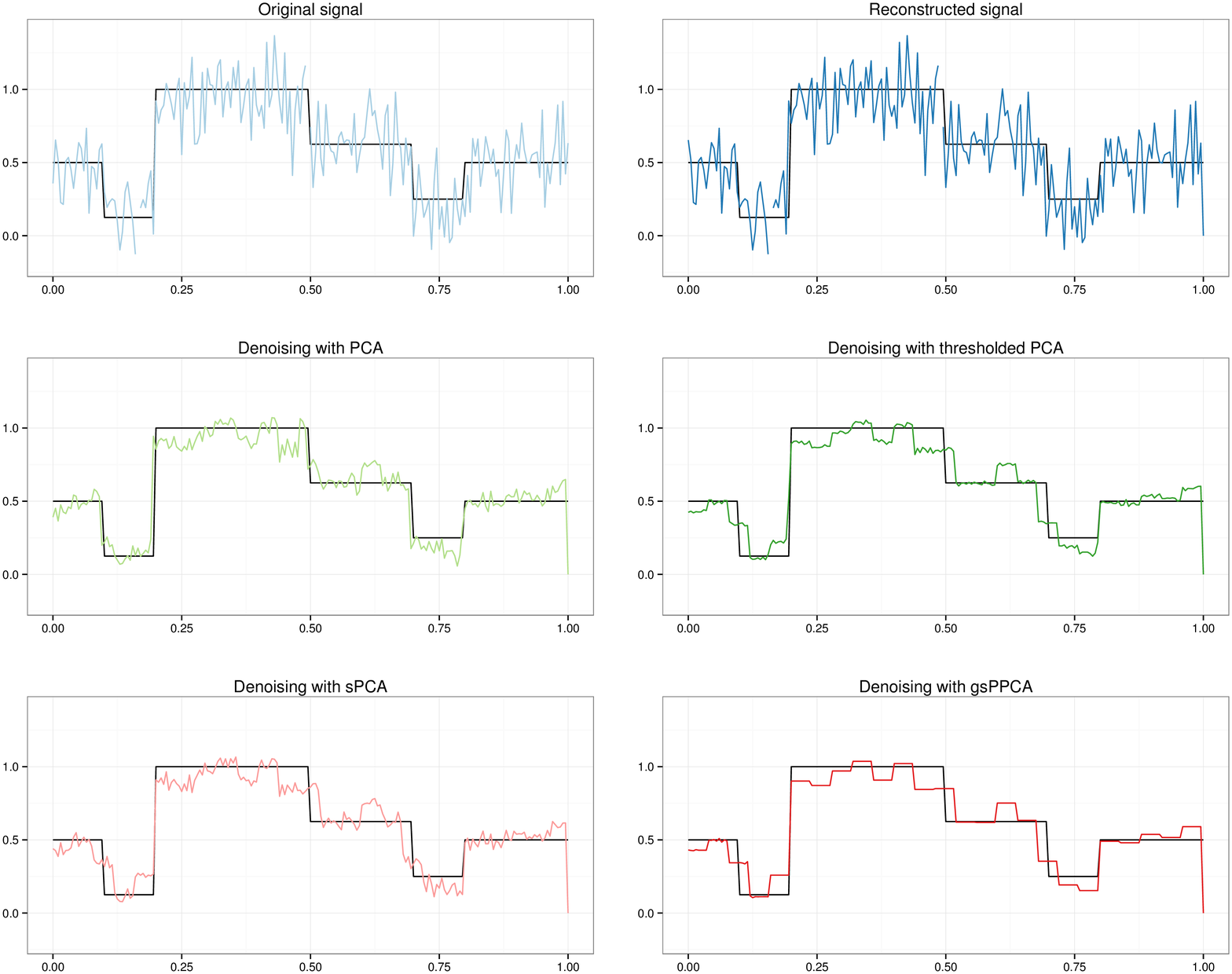}
\includegraphics[width=0.87\columnwidth]{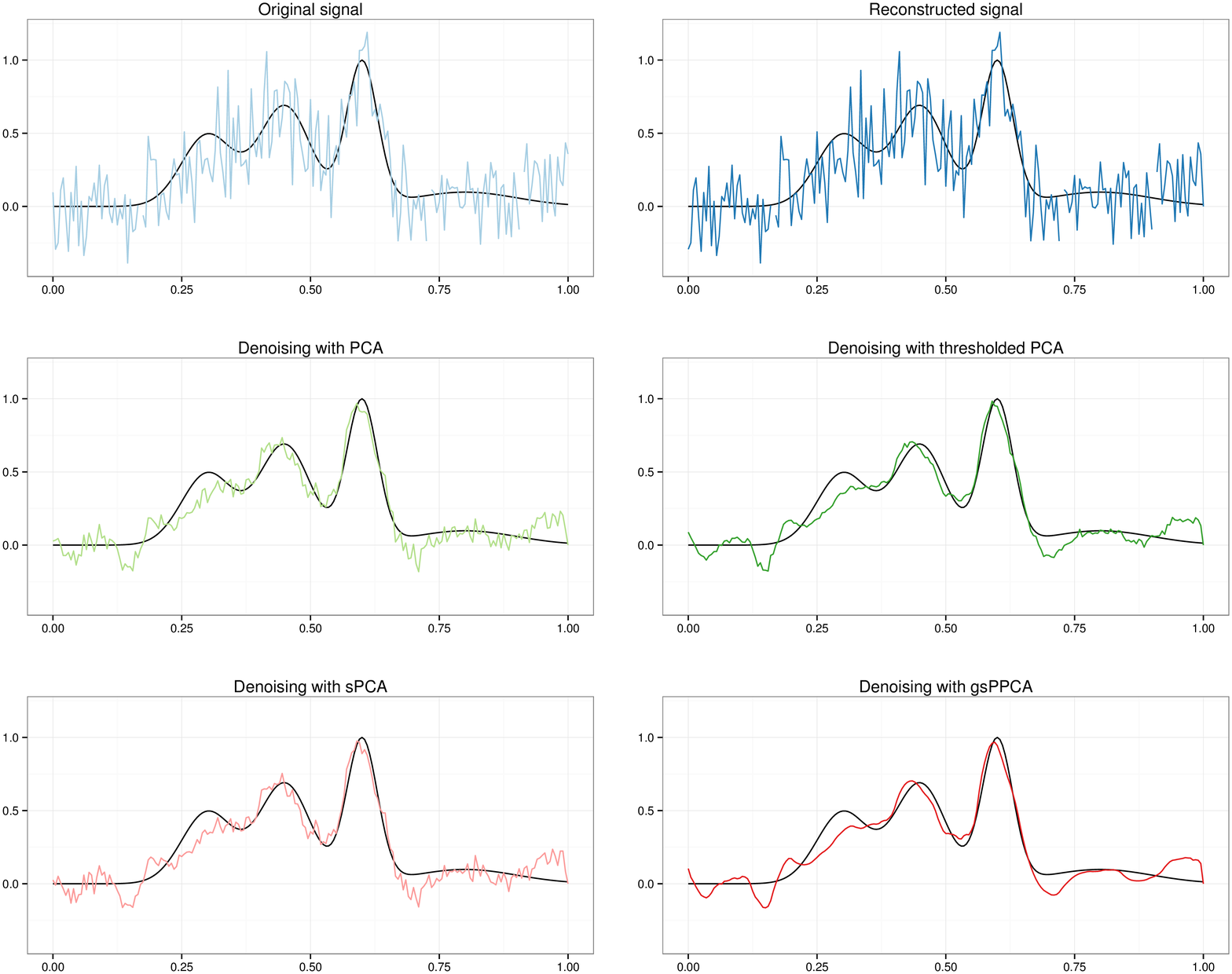}
\par\end{centering}

\protect\caption{\label{wav-2}Denoising results for signals A (top) and B (bottom) with PCA, thresholded PCA,
SPCA and GSPPCA.}

\end{figure}

\begin{table}
\begin{centering}
\begin{tabular}{l|ccccc}

Scenario & Wavelet & PCA & tPCA & SPCA & GSPPCA\tabularnewline
\hline 
A & 9.516$\pm$0.819 & 2.719$\pm$0.439 & 2.484$\pm$0.372 & 2.480$\pm$0.371 & \textbf{2.283$\pm$0.344}\tabularnewline
\hline 
B & 8.156$\pm$0.725 & 1.390$\pm$0.351 & 1.253$\pm$0.343 & 1.406$\pm$0.354 & \textbf{1.193$\pm$0.337}\tabularnewline

\end{tabular}
\par\end{centering}

\protect\caption{\label{wav-4}Reconstruction error (sum of squared errors) for wavelet
signal denoising on the two simulation scenarios (results are averaged
on 50 signal reconstructions). Standard deviations are also provided.}
\end{table}


\section{Application to unsupervised gene selection}

Considering again the breast cancer data set previously studied in Section 3, we address here the issue of the biological significance of the selected genes. To this end, we will use the \emph{pathway enrichment index} (PEI) introduced by \cite{teschendorff2007} and used in a sparse PCA framework by \citep{journee2010}.

\subsection{Pathway enrichment as a measure of biological significance}

In this subsection, we briefly review how the PEI can be computed in order to evaluate the quality of a given subset of genes. For more details on the PEI, see \cite{teschendorff2007} or \cite{journee2009}, and on hypergeometric tests and enrichment, see \cite{rivals2007}.

Suppose that using a microarray data matrix $\mathbf{X} \in \mathbb{R}^{n \times p}$ where each variable corresponds to a gene, an algorithm infers a subset $\mathbf{s} \subset \{1,...,p\}$ of genes. A way to assess its biological significance is to compare $\mathbf{s}$ to many other subsets \emph{which are known to be biologically relevant}. In this case, the biologically relevant subsets are defined by \emph{biological pathways}, and are therefore groups of genes involved in series of biochemical reactions linked to a certain biological function. Let us denote these known subsets  $\mathbf{b}_1,...,\mathbf{b}_N \subset \{1,...,p\}$. For our breast cancer experiment, we use the $N=1116$ pathways from the Reactome database \citep{fabregat2016} included in the R package \texttt{reactomePA} \citep{yu2016}. For $k\leq N$, the \emph{enrichment} of $\mathbf{s}$ in the $k$-th pathway of this list is the statistical significance of its overlap with $\mathbf{b}_k$, evaluated using the \emph{hypergeometric test}. More specifically, for each $k\leq N$, the null hypothesis of this test is that the genes in $\mathbf{s}$ are chosen uniformly at random from the total gene population. Under this hypothesis, the test statistic ${\# (\mathbf{s} \cap \mathbf{b}_k)}$ follows a hypergeometric distribution and a $p$-value can be computed to assess the statistical significance of the overlap. Because we are conducting one test for each pathway considered, these $p$-value are then adjusted using the Benjamini-Hochberg procedure to control the false discovery rate \citep{benjamini1995}. The subset $\mathbf{s}$ is eventually declared enriched for a certain pathway if the adjusted $p$-value of the corresponding hypergeometric test is lower than $0.01$. The PEI is finally defined as the percentage of enriched pathways in the Reactome family.

\subsection{Results}

We compare in Table \ref{PEI} the PEI obtained by GSPPCA with $d=10$, SPCA and thresholded PCA for several fixed cardinalities. Similarly to \cite{zou2006}, the two local methods are computing a single sparse axis.
As in \cite{journee2010} SPCA appears to give slightly better results than thresholded PCA.
GSPPCA significantly outperforms the two other methods.  This means that the genes selected by GSPPCA are consistently more associated with the Reactome pathways, and are therefore more interpretable. This highlights the fact that projecting the data onto a globally sparse subspace of dimension higher than one leads to significantly more interpretable and biologically plausible results.
\begin{table}[t]
\caption{PEI for several fixed cardinalities} \label{PEI}
\centering
\begin{tabular}{llccc}

  \multicolumn{2}{l}{Cardinality}     & tPCA & SPCA & GSPPCA        \\ \hline
290 &\emph{selected by tPCA}  & 0.09    & 0.09 & \textbf{3.22} \\ \hline
1000 & & 1.88 & 1.88 & \textbf{4.57} \\ \hline
1965 & \emph{selected by GSPPCA} & 1.7 & 1.61 & \textbf{5.19} \\ \hline
3000 & & 1.16 & 1.43 & \textbf{3.58} \\ \hline
4466 & \emph{selected by SPCA} & 3.04 & 3.22 & \textbf{4.29} \\ \hline
5000 & \emph{selected by SPCA} & 1.79 & 1.88 & \textbf{2.42} \\
\end{tabular}
\end{table}
Regarding the estimation of the sparsity level, choosing the one that explains $99\%$ of the variance led SPCA to selecting $4810$ genes, which is difficult to interpret. For thresholded PCA, we selected the sparsity level using a criterion proposed by \cite{teschendorff2007}. Even though it led to the sparsest solution, {its PEI was very small. Regarding GSPPCA, the noiseless marginal log-likelihood and the PEI of the corresponding models are plotted on Figure~\ref{margenes}. We can see that the marginal likelihood peak corresponds to  highly interpretable genes: more than $5\%$ of the biological pathways in the Reactome family have a significant overlap with the genes selected by GSPPCA.
Furthermore, models with a lower marginal likelihood have generally a lower PEI. To a certain extend, this shows that our marginal likelihood expression can stand as an indicator of biological significance.}

\begin{figure}[tb]
   \caption{Marginal likelihood for the gene selection problem} \label{margenes}
   \includegraphics[width=0.9\columnwidth]{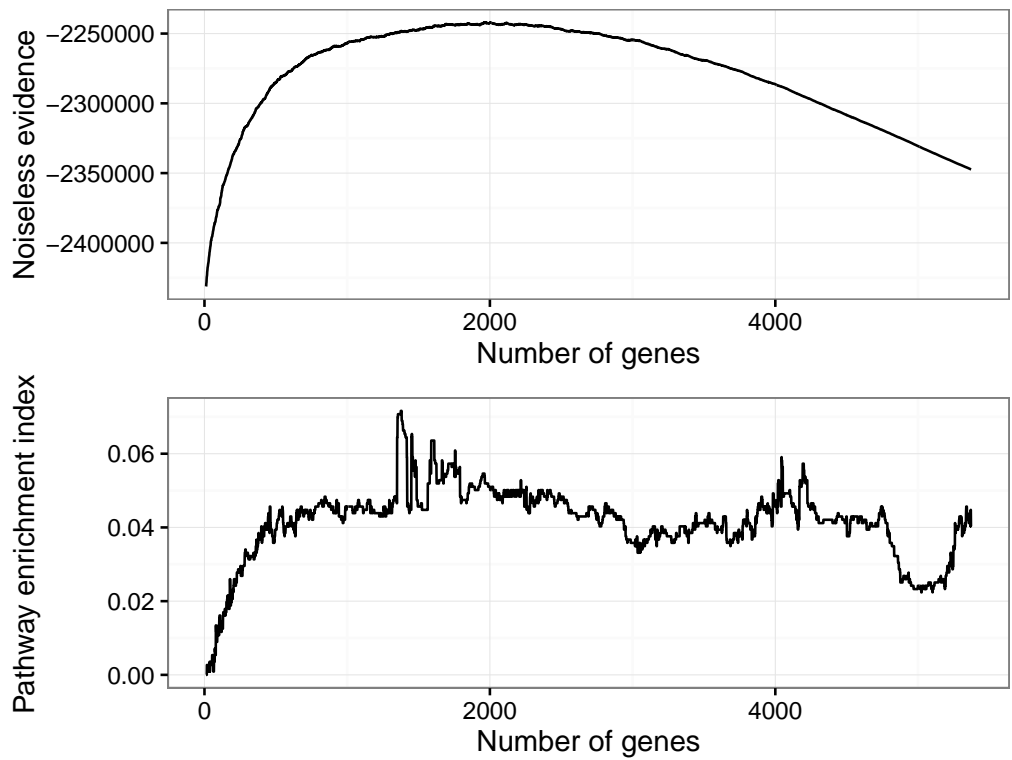}
\end{figure}


\section{Conclusion}

Unsupervised feature selection is an hazy and exciting problem. It becomes particularly difficult and ill-posed when no specific learning task (such as clustering) is driving it. We have proposed in this paper a new method for unsupervised feature selection based on the idea that the data may lie close to a subspace of moderate dimension spanned by a basis with a shared sparsity pattern. On several real data sets, this approach outperforms a popular method which consists in finding the sparsity pattern of the single leading principal vector of the data. These results suggest that, on many real-life high-dimensional data sets, an important part of the information cannot be captured by one-dimensional subspace approximations.

While building our framework, we derived the first closed-form expression of the marginal likelihood of a Bayesian PCA model, using the noiseless model of \cite{roweis1998}. Regarding future work, it would be interesting to see if more complex priors can be used and to what extend our expression can lead to a simultaneous estimation of the sparsity level and the dimension of the latent space. Indeed, intrinsic dimension estimation, which was beyond the scope of this paper, has an enduring relationship with probabilistic versions of PCA~\citep{minka2000,bouveyron2011,nakajima15} and would be an interesting direction.

\section*{Acknowledgements}
The authors would like to thank Magnus Ulfarsson for providing svnPCA software and Florentin Damiens for helpful discussions on Bessel functions. Part of this work was done while PAM was visiting University College Dublin, funded by the Fondation Sciences Math\'ematiques de Paris (FSMP).
\newpage
\appendix
\section*{Appendix A. Proof of Theorem 1}



\begin{proof}
Let us first consider the case where all variables are active and assume that $\mathbf{v}=(1,1,...,1)$. Therefore, $\mathbf{V}=\mathbf{I}_p$ and the considered model reduces to probabilistic PCA. In this framework, we will derive the density of $\mathbf{x}$ by computing the Fourier transform of its characteristic function.

In order to compute the characteristic function of $\mathbf{x}$, we first decompose the latent vector $\mathbf{y}$ in the canonical base $$\mathbf{y}=y_1\mathbf{e_1}+...+y_d\mathbf{e_d},$$ where $(\mathbf{e_i})_{i\geq d}$ is the canonical base of $\mathbb{R}^d$. We can now write the vector $\mathbf{W} \mathbf{y}$ as a sum of of $d$ i.i.d variables $$\mathbf{W} \mathbf{y} = y_1\mathbf{W}\mathbf{e_1} +...+y_d\mathbf{W} \mathbf{e_d}.$$ Its characteristic function will consequently be
$$\varphi_{\mathbf{W} \mathbf{y}}=(\varphi_{y_1\mathbf{W}\mathbf{e_1}})^d.$$ Now, for all $\mathbf{u}\in \mathbb{R}^d$, we have  
\begin{align}\varphi_{y_1\mathbf{W}\mathbf{e_1}}(\mathbf{u}) &= \mathbb{E}[\exp( i y_1\mathbf{e_1}^T\mathbf{W}^T\mathbf{u})] \\ &=\mathbb{E}\left[\exp\left( i y_1 \sum_{k=1}^p w_{k1}u_k\right)\right], \end{align}
but, since $w_{st}\sim \mathcal{N}(0,\alpha)$ for all $s,t$, we will have $$\frac{1}{\sqrt{\alpha}||\mathbf{u}||_2}\sum_{k=1}^p w_{k1}u_k \sim \mathcal{N}(0,1),$$
thus, since $\mathbf{y}$ and $\mathbf{W}$ are independent, the law of $(\sqrt{\alpha}||\mathbf{u}||_2)^{-1}y_1 \sum_{k=1}^p w_{k1}u_k$ will be the one of a product of two standard Gaussian random variables, whose density is $1/\pi K_{0}(| . |)$ \citep{Wishart}. Therefore, we find that
\begin{align*}
\varphi_{y_1\mathbf{W}\mathbf{e_1}}(\mathbf{u}) &= \frac{1}{\pi}\int_{-\infty}^{+\infty} K_0(|t|)e^{i \sqrt{\alpha}||\mathbf{u}||_2 t}dt \\ &= \frac{2}{\pi} \int_{0}^{+\infty} K_0(t) \cos(\sqrt{\alpha}||\mathbf{u}||_2 t)dt,\end{align*}
is simply the cosine Fourier transform of a univariate Bessel function. Using a formula in \citet[p. 486]{AS}, we eventually find that $$\varphi_{y_1\mathbf{W}}(\mathbf{u}) = \frac{1}{\sqrt{1+\alpha ||\mathbf{u}||_2^2}},$$ which leads to $$\varphi_{\mathbf{W} \mathbf{y}} (\mathbf{u})=  \frac{1}{(1+\alpha ||\mathbf{u}||_2^2)^{d/2}}.$$
Finally, since the noise term and $\mathbf{W} \mathbf{y}$ are independent, the characteristic function of $\mathbf{x}$ will be $$\varphi_\mathbf{x}(\mathbf{u})=\varphi_{\mathbf{W} \mathbf{y}} (\mathbf{u})\varphi_{\boldsymbol{\varepsilon}}(\mathbf{u})=\frac{e^{-\sigma^2 ||\mathbf{u}||_2^2}}{(1+\alpha ||\mathbf{u}||_2^2)^{d/2}}.$$
The density of $\mathbf{x}$ is then given by the Fourier transform of its characteristic function
$$p(\mathbf{x})=\frac{1}{(2\pi)^p}\int_{\mathbb{R}^p}\varphi_\mathbf{x}(\mathbf{u})e^{i \mathbf{x}^T\mathbf{u}}d\mathbf{u},$$but, since $\varphi_\mathbf{x}(\mathbf{u})$ is a radial function (i.e a function that only depends on the norm of its argument), its Fourier transform can be expressed as a univariate integral \citep{operators} and we can write \begin{equation}p(\mathbf{x})=\frac{||\mathbf{x}||_2^{1-p/2}}{(2\pi)^{p/2}}\int_0^{+\infty}\frac{u^{p/2}e^{-\sigma^2 u^2}}{(1+\alpha u^2)^{d/2}}J_{p/2-1}(u||\mathbf{x}||_2)du, \label {toutactif} \end{equation}
which is the desired form for the case with no inactive variable.

In the general case, $\mathbf{v}$ is not necessarily equal to $(1,1,...,1)$ but we can notice that, since $\mathbf{x}_\mathbf{v}$ and $\mathbf{x}_{\mathbf{\bar v}}$ are independent, we can write $p(\mathbf{x})=p(\mathbf{x}_{\mathbf{\bar v}})p(\mathbf{x}_\mathbf{v})$. Applying \eqref{toutactif} to $\mathbf{x}_\mathbf{v}$ allows us to compute $p(\mathbf{x}_\mathbf{v})$ and to eventually obtain the expression of the density given by the theorem. \end{proof}

\section*{Appendix B. Proof of Theorem 2}

We begin by proving the following lemma, which links the distribution of the product between a Gaussian matrix and a Gaussian vector with the Bessel distribution. This result may be of independent interest.

\begin{lemma}
Let $\mathbf{A}$ be a $q \times d$ random matrix such that $a_{ij} \sim \mathcal{N}(0,s^2)$ with $s>0$ for all $i,j$ and let $\mathbf{b} \sim \mathcal{N}0,\mathbf{I}_d)$. Then $\mathbf{Ab}$ follows a Bessel distribution with parameters $s$ and $(d-q)/2$.
\end{lemma}
\begin{proof}
Using the decomposition arguments from the proof of Theorem 1, the characteristic function of $\mathbf{Ab}$ is, for all $\mathbf{u} \in \mathbb{R}^k$,
$$\varphi_{\mathbf{Ab}}(\mathbf{u})=\frac{1}{(1+||\mathbf{u}||^2_2/s)^{d/2}},$$
which is exactly the characteristic function of the symmetric multivariate Bessel distribution \citet[Def. 2.5]{fang}.
\end{proof}
We can now prove Theorem 2.
\begin{proof}
Let us first consider the case where all variables are active and assume that $\mathbf{v}=(1,1,...,1)$. Using L\'evy's continuity theorem, $\boldsymbol{\varepsilon_2}$ weakly converges to zero when $\sigma_2$ vanishes. Since zero is a constant, this convergence also happens to be in probability \cite[p. 10]{van2000}. The variable $\mathbf{x}$ therefore converges in probability to $\mathbf{W}\mathbf{y}$, which follows a $\textup{Bessel}(1/\alpha,(d-q)/2)$ distribution according to our lemma.

In the general case when $\mathbf{v}$ is not necessarily equal to $(1,1,...,1)$ we can prove \eqref{noiseless} by invoking the independence between  $\mathbf{x}_\mathbf{v}$ and $\mathbf{x}_\mathbf{\bar{v}}$, similarly to the proof of Theorem 1. \end{proof}

\section*{Appendix C. Proof of Proposition 4}
\begin{proof}
Since a sum of concave functions is concave, it is sufficient to prove that the function
$g: \alpha \mapsto p(\tilde{\mathbf{x}}|\mathbf{v},\alpha,\sigma_1) $ is strictly concave.
Up to unnecessary additive constants, we have for all $\alpha >0$,
$$g(\alpha) = d \log \alpha \\ + \log \left((\alpha||\tilde{\mathbf{x}}_\mathbf{v}||_2)^{\frac{q-d}{2}} K_{\frac{q-d}{2}}\left(||\tilde{\mathbf{x}}_\mathbf{v}||_2\alpha\right) \right).$$
Using standard results about Bessel functions derivatives \citep[p. 376]{AS}, it can be shown that
$$g'(u) = \frac{d}{\alpha} - ||\tilde{\mathbf{x}}_\mathbf{v}||_2  h(u),$$ where the $h$ is the ratio $$h(\alpha)=\frac{K_{\frac{q-d}{2}-1}\left(||\tilde{\mathbf{x}}_\mathbf{v}||_2\alpha \right)}{K_{\frac{q-d}{2}}\left(||\tilde{\mathbf{x}}_\mathbf{v}||_2\alpha \right)}.$$ As proven independently by \cite{Lorch1967} and \cite{hartman1974}, since $q-d\geq 0$, $h$ is a increasing function on $\mathbb{R}^*_+$. Therefore $g'$ is stricly decreasing and $g$ is strictly concave.\end{proof}

\section*{Appendix D. Proof of Proposition 5}

\begin{proof}
\emph{Variational distribution of the latent vectors.} Using a standard result in variational mean-field approximations~\citep[chap. 10]{bishop2006}, we can write $$\ln q^*(\mathbf{y})=\mathbb{E}_{q(\mathbf{W})}[\ln p(\mathbf{X},\mathbf{Y},\mathbf{W}|\boldsymbol{\theta})]$$ which leads to the factorization
$q^*(\mathbf{y})=\prod_{i \leq n} q^*(\mathbf{y}_i).$
Then, for each $i\leq n$, we can write, up to unnecessary additive constants, $$\ln q^*(\mathbf{y}_i)=\mathbb{E}_{q(\mathbf{W})}[\ln p(\mathbf{x}_i,\mathbf{y}_i,\mathbf{W}|\boldsymbol{\theta})]=\mathbb{E}_{q(\mathbf{W})}\left[\frac{-1}{2\sigma^2}||\mathbf{x}_i-\mathbf{UWy}_i||_2^2\right]-\frac{1}{2}||\mathbf{y}_i||_2^2,$$ thus
$$\ln q^*(\mathbf{y}_i)= \frac{-1}{2\sigma^2} \mathbf{y_i}^T\mathbb{E}_{q(\mathbf{W})}[\mathbf{W}^T\mathbf{U}^2\mathbf{W}]\mathbf{y}_i+\frac{1}{\sigma^2}\mathbf{y}_i^T\mathbb{E}_{q(\mathbf{W})}[\mathbf{W}]^T\mathbf{U}\mathbf{x}_i-\frac{1}{2}||\mathbf{y}_i||_2^2,$$ which leads to the desired form.

\emph{Variational distribution of the loading matrix.} Similarly, up to unnecessary additive constants,
$$\ln q^*(\mathbf{W}) = \frac{-1}{2\sigma^2}\sum_{i=1}^n\mathbb{E}_{q(\mathbf{y}_i)}[||\mathbf{x}_i-\mathbf{UWy}_i||_2^2]-\frac{\alpha^2}{2}\sum_{i=1}^p||\mathbf{w}_i||_2^2,$$
$$\ln q^*(\mathbf{W}) =\sum_{i=1}^n\left( \frac{-1}{2\sigma^2}\sum_{j=1}^p u_j^2\mathbf{w}_j^T\mathbb{E}_{q(\mathbf{y}_i)}[\mathbf{y}_i\mathbf{y}_i^T]\mathbf{w}_j + \frac{1}{\sigma^2}\sum_{j=1}^px_{i,j} u_j \mathbf{w}_j^T\mathbb{E}_{q(\mathbf{y}_i)}[\mathbf{y}_i] \right)-\frac{\alpha^2}{2}\sum_{i=1}^p||\mathbf{w}_i||_2^2,$$
and
$$\ln q^*(\mathbf{W}) =\sum_{i=1}^p\left( \frac{-1}{2\sigma^2}\sum_{j=1}^p u_j^2\mathbf{w}_j^T\mathbb{E}_{q(\mathbf{y}_i)}[\mathbf{y}_i\mathbf{y}_i^T]\mathbf{w}_j + \frac{1}{\sigma^2}\sum_{j=1}^px_{i,j} u_j \mathbf{w}_j^T\mathbb{E}_{q(\mathbf{y}_i)}[\mathbf{y}_i] \right)-\frac{\alpha^2}{2}\sum_{i=1}^p||\mathbf{w}_i||_2^2,$$
which leads to the factorization $q^*(\mathbf{W})=\prod_{j \leq p} q^*(\mathbf{w}_i)$ and to the desired expression. \end{proof}

\section*{Appendix E. Proof of Proposition 6}

\begin{proof}By definition, we have
$$-\mathcal{F}_q(\mathbf{x_1},...\mathbf{x_n}|\boldsymbol{\theta})=\mathbb{E}_q[\ln p(\mathbf{X},\mathbf{Y},\mathbf{W}|\boldsymbol{\theta})]+H(q),$$ therefore \begin{multline*}-\mathcal{F}_q(\mathbf{x_1},...\mathbf{x_n}|\boldsymbol{\theta})=-np\ln \sigma - \frac{1}{2\sigma^2}\Tr (\mathbf{X}^T\mathbf{X})- \frac{1}{2\sigma^2}\sum_{i=1}^n\mathbb{E}_q[\mathbf{y}_i\mathbf{W}^T\mathbf{U}^2\mathbf{W}\mathbf{y}_i]+ \frac{1}{\sigma^2}\sum_{i=1}^n\mathbf{x}_i^T\mathbf{UM}\boldsymbol{\mu}_i \\  + \sum_{k=1}^p\left(d \ln \alpha -\frac{\alpha^2}{2}\mathbb{E}_q[\mathbf{w}_k^T\mathbf{w}_k] \right)- \frac{1}{2}\sum_{i=1}^n \mathbb{E}_q[\mathbf{y}_i^T\mathbf{y}_i] +\frac{n}{2}\ln |\mathbf{\Sigma}|+\frac{1}{2}\sum_{k=1}^p\ln |\mathbf{S}_k|, 
\end{multline*}
and computing the expectations leads to
\begin{multline}-\mathcal{F}_q(\mathbf{x_1},...\mathbf{x_n}|\boldsymbol{\theta})=-np\ln \sigma + dp \ln \alpha - \frac{1}{2\sigma^2}\Tr (\mathbf{X}^T\mathbf{X})- \frac{1}{2\sigma^2}\sum_{i=1}^n\sum_{k=1}^pu_k^2\Tr[(\boldsymbol{\Sigma}+\boldsymbol{\mu}_i\boldsymbol{\mu}_i^T)(\boldsymbol{S}_k+\mathbf{m}_k\mathbf{m}_k^T)]\\+ \frac{1}{\sigma^2}\sum_{i=1}^n\mathbf{x}_i^T\mathbf{UM}\boldsymbol{\mu}_i   + \sum_{k=1}^p -\frac{\alpha^2}{2}\Tr(\mathbf{S}_k+\mathbf{m}_k\mathbf{m}_k^T)-\frac{1}{2}\sum_{i=1}^n \Tr(\boldsymbol{\Sigma}+\boldsymbol{\mu}_i\boldsymbol{\mu}_i^T) +\frac{n}{2}\ln |\mathbf{\Sigma}|+\frac{1}{2}\sum_{k=1}^p\ln |\mathbf{S}_k|, 
\end{multline}
which allows us to conclude.
\end{proof}
\section*{Appendix F. Details about the breast cancer data set}

{The microarray data set used in this paper is included in the \texttt{breastCancerVDX} R package \citep{VDX} and contains the gene expression data published by \cite{wang2005} and \cite{minn2007}. It contains expression levels of $22283$ probes for 344 patients. In order to be able to provide an interpretation of feature selection, we reduced the data from probe-level to gene-level using the following procedure:\begin{itemize}\item first, the probes with no gene identifier were discarded
\item then, the data was aggregated to gene-level using the \texttt{collapseRows} R function of \cite{miller2011}, \item among the genes obtained, only the genes listed in the Reactome database \citep{fabregat2016} were kept in order to eventually perform pathway enrichment, \item finally, the data was centered but not standardized.\end{itemize} The resulting data matrix contains $5391$ variables (genes) and 344 observations (patients).}

\vskip 0.2in
\bibliography{biblio}

\end{document}